\documentclass[10pt]{article}

\usepackage{amsmath,amssymb,amsthm}
\usepackage{hyperref}
\usepackage{url}
\usepackage{tikz}
\usetikzlibrary{decorations.pathreplacing}

\newtheorem{defn}{Definition}
\newtheorem{theorem}{Theorem}
\newtheorem{corollary}{Corollary}
\newtheorem{problem}{Problem}
\newtheorem{lemma}{Lemma}

\newtheorem{Rem}{Remark}
\numberwithin{equation}{section}
\newtheorem*{conj}{Conjecture}

\title{Pair Correlation Factor and the Sample Complexity of Gaussian Mixtures}

\author{Farzad Aryan} 

\date{}

\begin{document}

\maketitle

\begin{abstract}
We study the problem of learning Gaussian Mixture Models (GMMs) and ask: which structural properties govern their sample complexity? Prior work has largely tied this complexity to the minimum pairwise separation between components, but we demonstrate this view is incomplete.\\

We introduce the \emph{Pair Correlation Factor} (PCF), a geometric quantity capturing the clustering of component means. Unlike the minimum gap, the PCF more accurately dictates the difficulty of parameter recovery.\\

In the uniform spherical case, we give an algorithm with improved sample complexity bounds, showing when more than the usual $\epsilon^{-2}$ samples are necessary.
\end{abstract}

\section{Introduction}
Gaussian Mixture Models have been extensively studied in the field of machine learning. Let $\mathcal{N}$ denote the probability density function (PDF) of a multivariate Gaussian distribution. A Gaussian mixture with $k$ components is given by: 

\begin{equation}
\label{gamm}
\Gamma= \sum_{i=1}^{k}\omega_i \mathcal{N}(\mu_i, \Sigma_i),
\end{equation}

where $\omega_i$ are the mixing weights, ($\sum \omega_i=1$), $\mu_i$ are the means, and $\Sigma_i$ are the covariance matrices. \\

For example, human height can often be modeled as a mixture of two Gaussians, corresponding to two genders. If we further condition on $k/2$ different racial groups, the data can instead be modeled as a mixture of $k$ Gaussians. \\

Now, let's assume we are given a sequence of i.i.d. samples, $x_1, \cdots, x_n$, generated from such a mixture $\Gamma$—for instance, measurements of human height. Our task is to estimate the parameters of this mixture (the weights, means, and variances of its Gaussian components) based on the observed data. A central question in the study of Gaussian mixture models is: How many samples are needed to learn these parameters with a desired level of accuracy? This quantity, known as the sample complexity, lies at the heart of our investigation.\\

The problem of learning a Gaussian mixture is often approached in two ways. The first focuses on approximating the entire distribution $\Gamma$ itself—typically measured in terms of total variation distance \cite{AMB}. This method ensures closeness of distributions but does not directly recover the underlying parameters. The second, which is the focus of this paper, seeks to estimate each parameter of $\Gamma$—namely $\omega_i$, $\mu_i$, and $\Sigma_i$—with a specified accuracy. To make this precise, we first clarify what we mean by “good precision.”

\begin{defn}
For a given $\Gamma$ as in \eqref{gamm}, an algorithm is said to learn the parameters of $\Gamma$ with $\epsilon$-precision if, for every $\mu_i, \Sigma_i,$ and $\omega_i$ in \eqref{gamm}, it produces $\hat{\mu}_i, \hat{\Sigma}_i$, and $\hat{\omega}_i$ such that $|\mu_i- \hat{\mu}_i| < \epsilon$, $|\Sigma_i-\hat{\Sigma}_i|<\epsilon$, and $|\omega_i- \hat{\omega}_i| < \epsilon$.  \footnote{In higher dimension, $|\Sigma_i - \hat{\Sigma}_i|$ is measured in spectral norm.}

\end{defn}
\subsection{Main Contribution}
Let us first formally define the notion of sample complexity, and then outline the main contribution of our work.

\begin{defn}[Sample Complexity]
Given a Gaussian mixture $\Gamma$ as in \eqref{gamm}, the \emph{sample complexity} of learning $\Gamma$ with $\epsilon$-precision and confidence $1-\delta$ 
is the minimal integer $n=n(\epsilon,\delta,k)$ such that there exists an algorithm which, 
using $n$ i.i.d.\ samples from $\Gamma$, outputs estimates 
$\hat{\mu}_i,\hat{\Sigma}_i,\hat{\omega}_i$ satisfying the bounds in Definition~1 
with probability at least $1-\delta$.
\end{defn}

This paper addresses a fundamental question: 
\emph{What truly drives the sample complexity of Gaussian mixture models?}  
It is natural to think that if the means and variances of two Gaussians are very close, distinguishing between them becomes difficult. This intuition has led to the popular belief that the sample complexity of learning Gaussian mixtures is governed primarily by the minimum separation between component means \cite{Mo-Pr, annstat},  
\begin{equation}
\label{gmin}
    g_{\min} := \min_{1\leq i\neq j \leq k} |\mu_i - \mu_j|,
\end{equation}

Contrary to this belief, we identify a previously overlooked \emph{geometric factor} that drives the difficulty of learning Gaussian mixtures. We call it the \emph{Pair Correlation Factor (PCF)}. While the minimum gap offers only a crude worst-case measure, the PCF reflects the structural interactions among all component means and thus provides a far more accurate explanation of sample complexity.\footnote{For intuition on why PCF governs sample complexity, see Appendix A (A Thought Experiment).} Moreover, the use of PCF significantly reduces the required sample complexity in many instances, as we will demonstrate later with concrete examples. Finally, our results are optimal: in Section 2.4, we show that they match fundamental lower bounds, confirming that no further improvement is possible in this case.\\

Despite more than a century of research on Gaussian mixture models, 
to the best of our knowledge the role of the Pair Correlation Factor (PCF) 
as a governing structural quantity for sample complexity 
has not been identified before. \\

\textbf{Outline of the Paper. } The next subsection (§1.2) briefly reviews the historical background, including Pearson’s original work and Hardt and Price’s landmark results for two-component mixtures. In §1.3, we introduce the Pair Correlation Factor (PCF), and we state our main theorem. \\

In Section 2, we compare our results with prior work and highlight the improvements they offer. We begin by contrasting our approach with Hardt and Price’s analysis for $k=2$ mixtures and with Yang and Wu’s results for the spherical case, showing concrete scenarios where our bounds reduce the sample complexity from $\epsilon^{-8k/3}$ to $\epsilon^{-6}$. We then discuss the practical advantages of using the Pair Correlation Factor, especially under natural spacing distributions like Poisson, and explain why the uniform spherical case—though seemingly restrictive—captures the hardest setting, as also observed by Hardt and Price. After reviewing related algorithmic approaches (e.g., Liu–Li and Qiao et al.), we turn in §2.4 to a more philosophical question: what fundamentally drives the worst-case sample complexity? There, we examine lower bounds, arguing that $\epsilon^{-2k}$ samples are sufficient in the uniform spherical case and nearly optimal. Finally, we conclude the section with a conjecture regarding the sample complexity in the general GMM setting.

\subsection{Historical Background}
Parameter estimation for Gaussian Mixture Models (GMMs) dates back to the 18th century and was initially introduced by Pearson \cite{Pea}. Pearson analyzed a population of crabs and discovered that a mixture of two Gaussian distributions accurately explained the sizes of the crabs' foreheads. Since then, this topic has been extensively studied from various perspectives \cite{Mo-Pr, MGA, Lili, Mo-Va, Mo-Va-Kalai, Be-Si}. \\

Hardt and Price~\cite{Mo-Pr} further explored the case of $k=2$, originally studied by Pearson. They demonstrated that $\epsilon^{-12}$ samples are both necessary and sufficient for accurately estimating the parameters of a 2-component mixture with $\epsilon$-precision. However, when the means of the Gaussian components are well-separated, only $\epsilon^{-2}$ samples are required.  
This trade-off between $\epsilon^{-2}$ and $\epsilon^{-12}$ samples presents an interesting phenomenon: while $\epsilon^{-12}$ samples are necessary in the worst-case scenario, the required sample size can be significantly lower when the components are more distinguishable.\\

To provide further insight, Hardt and Price in \cite{Mo-Pr}[Theorem 3.10] established a connection between the sample complexity of a $2$-component mixture and the distance between the means of the two components. Consider the mixture 
\begin{equation*}
\Gamma= \omega_1 \mathcal{N}(\mu_1, \sigma_1)+ \omega_2 \mathcal{N}(\mu_2, \sigma_2),
\end{equation*} 
where $\sigma_1, \sigma_2 \in \mathbb{R}$ and $\omega_1+\omega_2=1.$ They devised an algorithm that, given 
\begin{equation*}
\gg 10^{4}|\mu_1- \mu_2|^{-12}
\end{equation*} 
samples, returns estimates $\hat{\mu}_1$ and $\hat{\mu}_2$ such that, for $i=1,2$,
\begin{equation*}
|\hat{\mu}_i - \mu_i| < 0.01|\mu_1 - \mu_2|.
\end{equation*}
If the difference between the means, $\mu_1 - \mu_2$, is proportional to $\epsilon$, it becomes clear how the factor $\epsilon^{-12}$ enters into the sample complexity.\\

In this paper, our objective is to generalize this result from a $2$-mixture to a $k$-mixture. The first significant challenge we encounter is how to extend the concept of the difference between means. In a $2$-mixture, there are only two means, resulting in a single gap between them. However, for a $k$-mixture, there are $k^2$ potential gaps between means. The impact of these gaps on the sample complexity has not been thoroughly investigated. This leads to an important question: Are these gaps deterministic factors in the sample complexity, and what role do the variances play?\\

As mentioned earlier, in some cases of $2$-mixtures, the inverse of the minimum gap raised to the power of $12$ appears in the sample complexity. For $k$-mixtures, Yang and Wu \cite{annstat} considered the spherical case where all the variances are equal to the identity matrix and showed that $(g_{\text{min}})^{-4k}$ contributes to the sample complexity. However, an open question remains: What is the precise impact of the minimum gap on the sample complexity?\\

To address this and extend Hardt and Price's results, we introduce the concept of the pair correlation factor. As we will demonstrate, this new approach significantly reduces the sample complexity in various instances of the problem.

\subsection{Pair correlation of means and sample complexity}

We now formally define the Pair Correlation Factor (PCF), a structural complexity measure that captures the true geometric difficulty of learning Gaussian mixtures. For a mixture with means $\mu_1, \dots, \mu_k$, the PCF at $\mu_m$ is defined as
\begin{equation}
\label{PairC}
 \mathcal{P}(\mu_m):= \prod_{\substack{n=1 \\ n \neq m}}^{k} |\mu_m- \mu_n|.
\end{equation}

At its core, the pair correlation factor at $\mu_m$ is a measure of how many means are in close vicinity of $\mu_m$. \\

Next, we state our main theorem. It establishes the sample complexity and provides an algorithm in the uniform spherical case, where $\omega_i=1/k$ and $\sigma_i=1$. Although this setting may appear restrictive, it actually represents one of the hardest cases for learning mixtures. When the variances of different components are very close, distinguishing between them becomes particularly challenging. In fact, Hardt and Price showed that mixtures with nearly equal variances are significantly harder to learn than those with more distinct variances. Thus, the uniform spherical case serves as a canonical example of the intrinsic difficulty of the problem, and results obtained here provide especially strong guarantees. Later, we will explain how our approach extends more broadly.
\begin{theorem}
\label{1stthm}
Let $\Gamma= \tfrac{1}{k}\sum_{m=1}^{k} \mathcal{N}(\mu_m, 1),$ with mean zero and the variance \footnote{Note that, the mean and variance of the whole mixture are not the same as those of its individual components.} of the whole mixture equal to $\sigma^2.$ Assume that we are given $\log\big(\frac{1}{\delta}\big)\epsilon^{-2}$ samples from $\Gamma$ with 
$$ 
\epsilon < \frac{\mathcal{P}(\mu_m)}{k^2(1+|\mu_m|^2)^{(k-1)/2}(2\sigma)^k e^{0.5(k/\sigma)^2}}.
$$ 
Then, with probability at least $1-\delta$, our algorithm returns $\hat{\mu}_m$ such that
 \begin{equation}
 \label{eqthm}
 |\hat{\mu}_m- {\mu}_m| < \frac{k^2(1+|\mu_m|^2)^{k/2}(2\sigma)^k e^{0.5(k/\sigma)^2}}{\mathcal{P}(\mu_m)} \epsilon.
\end{equation}
\end{theorem}

\begin{Rem}
As evident from the denominator in \eqref{eqthm}, a smaller $\mathcal{P}(\mu_m)$ results in a weaker bound on the right-hand side of the equation. Therefore, to preserve our precision, we need to pick a smaller $\epsilon$, and consequently, our algorithm requires more samples to estimate $\mu_m$. 
\end{Rem}

\begin{Rem}  
Regarding the assumption that the mixture has mean zero, note that we can always center the data by subtracting the empirical mean from each point. Since centering preserves pairwise gaps and covariances, this adjustment allows us to work without loss of generality in a zero-mean setting.
\end{Rem}

In essence, Theorem \ref{1stthm} demonstrates that the distribution of gaps between the $\mu_i$'s significantly impacts sample complexity. Furthermore, it provides a guide to distinguish between different cases of the problem. This connection represents a novel contribution within this paper. To highlight the strength and advantages of our result, we compare it with recent literature.\\

\section{How Our Results Improve on Prior Work}

We previously mentioned the work of Hardt and Price~\cite{Mo-Pr}, which focuses on mixtures of two Gaussians. In contrast, our approach extends to spherical mixtures with any number of components.\\

Closely related to our work and the approach by Hardt and Price~\cite{Mo-Pr} is the work of Yang and Wu \cite{annstat}. They considered the spherical case, where all variances are equal, and developed a method that returns a distribution close in Wasserstein distance to the original mixture. In their paper \cite{annstat}[Theorem 2], they consider scenarios in which a mixture has $k_0$ separated clusters. To illustrate the advantages of our results, let's assume $k_0=k/3$, with $k/3$ clusters separated by a difference of $1 + O(\epsilon)$. For example, the first cluster includes means $\{-\epsilon, 0, \epsilon\}$, the second cluster contains $\{1-\epsilon, 1, 1+\epsilon\}$, and the $k/3$-th cluster is composed of $\{k/3-\epsilon, k/2, k/3+\epsilon\}$. Each cluster consists of $3$ means. In this scenario, their algorithm requires more than $\epsilon^{-8k/3}$ samples to achieve $\epsilon$-precision, while our approach only needs $\epsilon^{-6}$ samples, which is a significant saving in sample complexity.\\

To explain $\epsilon^{-6}$ we state our first corollary: 
\begin{corollary}
\label{cor1}
Let $\Gamma$ be a uniform spherical mixture of $k$ Gaussian distributions with mean equal to zero and variance\footnote{Note that here $\sigma^2$ refers to the variance of the entire mixture, which in general differs from the variance of each individual component.} equal to~$\sigma^2$. Given
\begin{equation}
\label{eqcorr}
\log\big(\frac{1}{\delta}\big)\frac{\displaystyle{ c(\sigma, k)}}{\big | \min_{m} \displaystyle \mathcal{P}(\mu_m)\big |^2} \hspace{1 mm}\epsilon^{-2},
\end{equation} samples from $\Gamma$, where ~$ c(\sigma, k)= \big|k^2(1+k\sigma^{2})^{k/2}(2\sigma)^k e^{0.5(k/\sigma)^2}\big|^2$, with probability $1- \delta,$ we can approximate $\Gamma$'s parameters with $\epsilon$-precision.
\end{corollary}

Now in the example we previously mentioned, since we only get two consecutive gaps of length $\epsilon$ in each cluster, we have $|\mathcal{P}(\mu_m)|^2 > \epsilon^4$ for $1 \leq m \leq k$. Consequently, our sample complexity is upper-bounded by $\epsilon^{-6}$. \\

\subsection{Practical advantages} 
The difference between our methods is particularly significant in practice. Consider situations where the distribution of spacing between means resembles a Poisson process (which is a very natural assumption here). While small gaps are still present, their probability is exceedingly low. The mere possibility of such cases drastically increases the sample complexity when applying other methods by $(g_{\text{min}})^{-ck}$, with $c\geq 1,$ whereas in our approach, it only amplifies the sample complexity by $(g_{\text{min}})^{-kp}$, where $p$ stands for the probability of a series of consecutive gaps of length $g_{\text{min}}$, which in reality is often much smaller than $1$. This represents a substantial practical advantage.\\ 

Hardt and Price~\cite{Mo-Pr} showed that for 2-mixtures, learning is significantly harder when the component variances are nearly equal, while differing variances make the problem easier. This suggests why the spherical case, where all variances are equal, represents a particularly hard instance of the problem. Accordingly, the sample complexity we derive in this setting can serve as an upper bound for broader scenarios where variances are not identical. \\

Overall, our method could be beneficial in bounding the sample complexity in cases where there is some prior information on the distribution of gaps between means. Here we should mention that Yang and Wu~\cite{annstat} consider the non-uniform case where weights are not necessarily equal to $1/k,$ so in this sense, their setting is more general. The sample complexity they obtain, though, in almost all cases, is weaker than ours. \\

\subsection{Alternative Methods in the Uniform Spherical Case.}
So far, the results we have discussed rely on the method of moments. We now turn to other approaches in the uniform spherical setting that use different techniques. In particular, we highlight two recent works, both addressing special cases of our problem under the assumption that all means are well-separated.\\

Substantial separation between component means intuitively simplifies the task of clustering samples. In this vein, Liu and Li \cite{Lili} proposed an algorithm for spherical GMMs that runs in $poly(d,k)$ time, contingent on the condition  $$|\mu_i- \mu_j| > (\log k)^{1/2 + \varepsilon}$$  where $d$ denotes the dimension. Notably, when the separation between parameters exceeds $1$, our algorithm achieves $\epsilon$-accuracy with approximately $\epsilon^{-2}$ samples. \\

For scenarios where $d< \log k$, Qiao et al. \cite{MGA} further refined the results of Liu and Li, in the uniform spherical case. Their sample complexity though, with respect to $\epsilon$-precision, is quasi-polynomial i.e., $\epsilon^{-c\log(1 / \epsilon)}.$\\

What sets our method apart is the absence of assumptions regarding parameter separation, and the optimal bound on sample complexity subject to accuracy. Additionally, we clearly describe the circumstances under which more than $\epsilon^{-2}$ samples—a theoretical minimum—are required.\\

\subsection{Higher Dimension} Using the idea of random projection we can bring the problem from multidimensional setting to a one-dimensional form. Here is a quote from\cite{annstat} explaining the issue: ``When components are allowed to overlap significantly, the random projection approach is also
adopted by \cite{Mo-Va-Kalai, Mo-Va, Mo-Pr}, where the estimation problem in high dimensions is reduced to that in one dimension, so that univariate methodologies can be invoked as a primitive." For further insight see section 7 of \cite{annstat}.\\

\subsection{A philosophy behind the sample complexity.} Let us explain a potential philosophy that derives the sample complexity in the worst case scenario of the problem. For $k=2$, there are $5$ parameters to learn. Exponent $12$ suggests that we need a factor of $\epsilon^{-2}$ samples to distinguish between each of these parameters. When this is done, with an extra $\epsilon^{-2}$ samples we can reach $\epsilon$-accuracy.\\

For arbitrary $k,$ there are $3k-1$ parameters to learn. Hardt and Price~\cite{Mo-Pr}, considered a $k$-mixture with $3k-2$ parameters \footnote{They assume that at least two of the parameters cannot be very close to each other, which saves a factor of $\epsilon^{-2}$ samples. See \cite{Mo-Pr} the paragraph before Theorem 2.10.} to learn, and showed that at least $\epsilon^{-6k+2}$ samples are necessary for this case. This leads us to anticipate the optimal lower bound, for the general case, to be $\epsilon^{-6k}$, which raises the following question: are $\epsilon^{-6k}$ i.i.d. samples adequate for determining the parameters of a GMM with $\epsilon$-precision?\\

Note that in the uniform spherical case with mean zero,  we have $k-1$ parameters to recover.Therefore, we expect $\epsilon^{-2k}$ samples to suffice. Our last corollary confirms this:
\begin{corollary}
\label{cor0}
Let $\Gamma$ be a uniform spherical mixture of $k$ Gaussian distributions with mean equal to zero and variance equal to~$\sigma^2$. Assume that $\epsilon < g_{\min}$, defined in \eqref{gmin}. Set 
\[
c(\sigma, k)= \big|k^2(1+k\sigma^{2})^{k/2}(2\sigma)^k e^{0.5(k/\sigma)^2}\big|^2.
\]
Then, given 
\[
10^4\,c(\sigma, k)\,\epsilon^{-2k}\,\log\!\left(\tfrac{1}{\delta}\right)
\] 
samples from $\Gamma$, with probability at least $1-\delta$, we can learn the parameters of $\Gamma$ to $\epsilon/100$-accuracy.
\end{corollary}

Corollary \ref{cor0} shows that, in the worst-case scenario, we require on the order of $\epsilon^{-2k}$ samples to achieve $\epsilon$-accuracy with high probability. This essentially matches \cite{annstat}[Theorem 1], which addresses the worst-case regime. To conclude the introduction, we present another example illustrating how our method distinguishes between different structural cases of the problem.\\

Assume that in our mixture, we have gaps of length $\epsilon$ between consecutive means. However, these gaps are isolated, meaning that if $\mu_{n+1}-\mu_n= \epsilon,$  then the adjacent gaps are significantly larger:  $\mu_{n+2}-\mu_{n+1} \gg 1,$ and $\mu_{n}-\mu_{n-1} \gg 1.$ Further assume that we are aiming to learn $\mu_i$'s with $\epsilon/100$-accuracy. In this situation our result shows having roughly $c(k, \sigma) \epsilon^{-4}$ samples will suffice. \\

Now let us alter the situation by allowing two consecutive gaps of length $\epsilon$ in our mixture .i.e. $\mu_{n+1}-\mu_n= \epsilon,$ and $\mu_{n}-\mu_{n-1} = \epsilon.$ As before, we assume that other surrounding gaps remain significantly larger: $\mu_{n+2}-\mu_{n+1} \gg 1,$ and $\mu_{n-1}-\mu_{n-2} \gg 1$.  Under this modified scenario, our required sample complexity rises to $c(k, \sigma) \epsilon^{-6}$. Applying other methods we require around $\epsilon^{-ck}$ to learn both mixtures.\\

\subsection{A Conjecture for the General Case}
\label{sec:conjecture}

We conjecture that the sample complexity of parameter recovery for general Gaussian mixtures is fundamentally governed by a structural quantity that incorporates both mean separation and variance differences. This extends the insight of Hardt and Price~\cite{Mo-Pr}—originally established for two-component mixtures—to arbitrary $k$ and higher dimensions.\\

To formalize this, we introduce a \emph{variance-aware pair correlation factor}, which modifies the original PCF to account for statistical distinguishability not only from nearby means but also from differing covariances. This unifies the geometric and spectral contributions to learnability.

\begin{defn}[Variance-Aware Pair Correlation Factor]
Let $\Gamma $ be a GMM. For each component $\mu_m$, we define the variance-aware pair correlation factor as:
\[
\tilde{\mathcal{P}}(\mu_m) := \prod_{\substack{j = 1 \\ j \ne m}}^k \max\left( |\mu_m - \mu_j|, \|\Sigma_m - \Sigma_j\|^{1/2} \right),
\]
where $\|\cdot\|$ denotes the spectral norm.\footnote{We use the spectral norm since it measures the largest directional variance, which directly affects the ability to distinguish components. }

\end{defn}

\begin{conj}[Variance-Aware PCF Governs Sample Complexity]
Let $\Gamma$ be a $k$-component Gaussian mixture in $\mathbb{R}^d$ with bounded weights and well-conditioned covariance matrices\footnote{ (i.e., all $\Sigma_i$ have bounded condition number).}. Then there exists an absolute constant $C$ such that any algorithm achieving $\epsilon$-precision in recovering the parameters of $\Gamma$ requires at least
\[
n \geq \frac{C}{\epsilon^2} \cdot \max_{m} \left( \frac{1}{\tilde{\mathcal{P}}(\mu_m)^2} \right)^3.
\]
\end{conj}

We believe that resolving this conjecture would provide a definitive understanding of the sample complexity of GMMs. It would unify the worst-case perspective (based on minimum gap) with a more refined structural view that explains observed difficulty in both theoretical and practical settings.

\medskip

\noindent
In the special case $k = 2$, this conjecture recovers the main result of Hardt and Price~\cite{Mo-Pr}, where the sample complexity scales like
\[
\left( \frac{1}{\max\left((\mu_1 - \mu_2)^2, |\sigma_1^2 - \sigma_2^2| \right)} \right)^6.
\]

For general $k$, Theorem~\ref{1stthm} resolves the conjecture in the uniform spherical case (where all $\Sigma_i$ are equal to identity matrix). Extending this result to arbitrary covariances likely requires new mathematical tools, from the theory of polynomial root sensitivity and spectral analysis. This highlights both the difficulty of the problem and the potential for progress at the interface of statistics, analysis, and algebra.

\subsection{Proof Overview}
Let $P$ be a polynomial of degree $k$ with roots corresponding to the parameters of our mixture; we will refer to $P$ as our ``parameter polynomial." Our goal is to approximate the coefficients of $P$ with good precision and relate this approximation to the number of samples used. \\

We employ the method of moments in conjunction with Newton's identity to derive a polynomial whose coefficients are close to those of $P$. Since we are using samples to calculate the moments of our mixture, we must account for the empirical error. Therefore, we need to assess how much of this error affects our approximation of the parameter polynomial. Following this process, we obtain a polynomial whose coefficients are within a certain distance from $P$. This distance is proportional to the number of samples used to calculate the empirical moments. \\

Finally, we must establish an argument that connects the similarity between the coefficients of two polynomials to the difference between their roots. To do this, we use a theorem from real analysis that measures this relationship.\\

\section{The Method of Moments} One way of approaching the problem of parameter estimation of GMM's, is through the method of moments, see \cite{Mo-Va, Mo-Pr}. For a one dimensional Gaussian we have the $r$-th moment equals 
\begin{equation}
\label{mom-exp}
\mathcal{M}_r(\mu, \sigma):= \frac{1}{\sigma \sqrt{2\pi}}\int x^r e^{-\tfrac{1}{2}\big(\tfrac{x-\mu}{\sigma}\big)^2}dx,
\end{equation}
which is a polynomial in terms of $\mu, \sigma$ and easily calculable.\\

For a $k$-mixture, to simplify the problem, assume that $\omega_i=1/k,$ and consider the one dimensional case that covariance matrices $\Sigma_i= \sigma_i$ are real numbers. Therefore, moments of the $k$-mixture are
\begin{align}
\label{system}
\notag  & M_1= \frac{1}{k}\sum_{i=1}^{k} \mathcal{M}_1(\mu_i, \sigma_i)= \frac{1}{k}\sum_{i=1}^{k} \mu_i \\
\notag & M_2= \frac{1}{k}\sum_{i=1}^{k} \mathcal{M}_2(\mu_i, \sigma_i) = \frac{1}{k}\sum_{i=1}^{k} \mu^2_i + \frac{1}{k}\sum_{i=1}^{k} \sigma^2_i \\
& \hspace{3.7 cm} \cdots \\
\notag & M_m= \frac{1}{k}\sum_{i=1}^{k} \mathcal{M}_m(\mu_i, \sigma_i)=\frac{1}{k}\sum_{i=1}^{k} \sum_{j=0}^{ \lfloor m/2 \rfloor} c_{m,j} \mu^{m-2j}_{i}\sigma^{2j}_{i}.
\end{align}
\\

Here we have $2k$ variables $\mu_1 \cdots, \mu_k$ and $\sigma_1 \cdots, \sigma_k$, and as many equations as we desire. Intuitively, to find our variables it suffices to look at first $2k$ moments. In other words, using moments, we can form a system of $2k$ equations and we try to solve it.\\


Given $n$ i.i.d. samples $x_1 \cdots x_n$ from our mixture, the $m$-th empirical moment is 
\begin{equation}
\label{hatM}
 \hat{M}_m= \tfrac{1}{n} \sum_{j=1}^{n} x^m_j.
\end{equation} Note that $\hat{M}_m$ is random variable with mean $M_i$ and variance $\sigma^{2m}/n$, where $\sigma^2$ is the variance of the $k$-mixture. By Chebyshev’s inequality we have 
\begin{equation}
    {P}\big[|\hat{M}_m - M_m|> r\frac{\sigma^m}{\sqrt{n}}\big] < \frac{1}{r^2}.
\end{equation}
Take $n \geq \epsilon^{-2}$, then the portion of samples for which we have $$|\hat{M}_m - M_m| > r \epsilon \sigma^m,$$ is bounded by $r^{-2}.$  This would suffice for our purposes, take for example $r=10$, then we know that for $99\%$ of samples we have 
$$|\hat{M}_m - M_m| < 10 \epsilon \sigma^m.$$ This can be improved by taking samples in groups, calculate $\hat{M}_m$ for each group and look at the median of these estimates. In  \cite{Mo-Price}[Lemma 3.2] Hardt and Price proved that given $n \gg \log(1/\delta) \epsilon^{-2},$ samples from the mixture, with probability $1-\delta$ we have:  
\begin{equation}
\label{boundonmom}
 |\hat{M_i}- M_i|< \epsilon \sigma^i.
\end{equation}
Going forward we will use \eqref{boundonmom}. \\

\subsection{System of Equations and Perturbation of Coefficients.} In practice, we often have to approximate these moments. Consequently, a relevant question arises: How robust is the system \eqref{system} when subjected to perturbations in its constant terms?

\begin{problem}
\label{prb}
Let $\mathcal{M}_r$ be defined as in \eqref{mom-exp}, and consider the following system of equations:
\begin{align}
\label{sys_off}
\notag & M_1 + \varepsilon_1 = \frac{1}{k}\sum_{i=1}^{k} \mathcal{M}_1(\mu_i, \sigma_i) \\
& \hspace{2 cm} \cdots \\
\notag & M_{2k} + \varepsilon_{2k} = \frac{1}{k}\sum_{i=1}^{k} \mathcal{M}_{2k}(\mu_i, \sigma_i).
\end{align}
We know that for $\varepsilon_1=\varepsilon_2= \cdots= \varepsilon_{2k}=0$, the system has a real solution:
$${X} = (\mu_1, \cdots, \mu_k, \sigma_1, \cdots, \sigma_k). $$
Now consider the system with perturbation $\varepsilon = (\varepsilon_1, \cdots, \varepsilon_{2k})$ subject to $\varepsilon_i \ll \epsilon \sigma^i$, where $\sigma>0$ is fixed. Let $X_\varepsilon = (\hat{\mu}_1, \cdots, \hat{\mu}_k, \hat{\sigma}_1, \cdots, \hat{\sigma}_k)$ be a solution to the system with perturbation $\varepsilon$. \\

 Is it possible to establish a bound such as
\begin{equation}
\label{bnd}
\|{X} - X_\varepsilon\|_2 \ll C_k \epsilon^{\alpha_k},
\end{equation}
where $C_k$ and $\alpha_k$ are constants that depend on $k$?
\end{problem}

We solve this problem in the context of the uniform spherical case. Our approach involves constructing a polynomial whose coefficients closely resemble those of our parameter polynomial. Subsequently, our objective is to demonstrate that the roots of this polynomial provide a robust approximation of the parameters we seek to learn.\\

Note that, In general, roots of polynomials are sensitive to the perturbation of the coefficients. A famous example is Wilkinson's polynomial:
$$\omega(x)= (x-1)(x-2) \cdots (x-20).$$
If the coefficient of $x^{19}$ decreases from $-210$ by $2^{-23}$ to $-210.0000001192$, then the root at $x = 20$ grows to $x \simeq 20.8$ 
\section{The Uniform Spherical Case}
In this section we prove some necessary lemmas that will help us recover parameters of the mixture. We assume that all variances are equal to $1.$ This assumption simplifies equation \eqref{system}. \\

\noindent Our final goal is to approximate the coefficients of the following polynomial: $$P(x):=\prod_{i=1}^{k}(x-\mu_i),$$
using information we get from \eqref{system}.\\

Define $$P_m(\mu_1, \cdots, \mu_k)= \mu^m_1+ \cdots+ \mu^m_k.$$
We can write each moments in \eqref{system}, in terms of $P_m$ and vice versa. For example 
\begin{align}
\label{syexmp}
 \notag  & M_1= \frac{1}{k} P_1(\mu_1, \cdots, \mu_k) \\
\notag & M_2= 1+ \frac{1}{k} P_2(\mu_1, \cdots, \mu_k) \\
& \hspace{3.7 cm} \cdots \\
\notag & M_m= \frac{1}{k} \sum_{j=0}^{ \lfloor m/2 \rfloor} c_{m,j} P_{m-2j}.
\end{align}

Let us precisely calculate $c_{m,j}$ first, then we move to approximate $P_m$ using empirical moments we obtain in \eqref{hatM}. 
\begin{lemma}
\label{lemm1}
For $1 \leq m$ and $j \leq  m/2$, we have that
\[
c_{m,j} = {m \choose 2j} (2j-1)!!,
\]
where $(2i-1)!!$ is the product of odd numbers less than or equal to $2i-1$. Moreover, we have
\begin{equation}
\mathcal{M}_m(\mu, 1) = \mu^m + \sum_{i=1}^{m/2} {m \choose 2i}  (2i-1)!! \mu^{m-2i},
\end{equation}
where $\mathcal{M}_m(\mu, 1)$ is introduced in \eqref{mom-exp}.
\end{lemma}
\begin{Rem}
\label{rem3}
 Using the above lemma we immediately have the expansion of $M_m$ in terms of $P_i:$
 \begin{equation}
 \label{expmom}
   k M_m= P_m + \sum_{i=1}^{m/2} {m \choose 2i} (2i-1)!! P_{m-2i}.
\end{equation}
 In the next lemma we show that 
 \begin{equation}
 \label{defPm}
 P_m= kM_m + k\sum_{i=1}^{m/2} (-1)^i {m \choose 2i} (2i-1)!! M_{m-2i}.
\end{equation} 
\end{Rem}

\begin{proof}[Proof of Lemma \ref{lemm1}] We have 
 \begin{align*}
 \mathcal{M}_m(\mu, 1) & = \frac{1}{ \sqrt{2\pi}}\int x^r e^{-\tfrac{1}{2}\big({x-\mu}\big)^2} \\ & =\frac{1}{ \sqrt{2\pi}}\int (x+\mu)^m e^{-\tfrac{1}{2}{x}^2} = \sum_{i} {m \choose i}\mu^{m-i}E_{}(x^i),
 \end{align*} 
 where $x \sim \mathcal{N}(0, 1)$.\\ 
 
 We use $E(zf(z))= E(f'(z))$, therefore we have $$E(x^m)= E(x x^{m-1})=(m-1) E(x^{m-2}).$$ Also note that $E(x^m)$ is zero if $m$ is odd. This will give the lemma. 
\end{proof}
\noindent We now define a new object that is a empirical approximation of $P_m.$
\begin{defn} Let $\hat{M}_m$ be as \eqref{hatM}, following \eqref{defPm}, we define 
\begin{equation}
\label{Phat}
    \hat{P}_m= k\hat{M}_m + k\sum_{i=1}^{m/2} (-1)^i {m \choose 2i} (2i-1)!! \hat{M}_{m-2i}.
\end{equation}  
\end{defn}
\begin{lemma}
\label{lemmaem}
We have that $P_m$ satisfy: 
 \begin{equation}
 \label{defpm1}
 P_m= kM_m + k\sum_{i=1}^{m/2} (-1)^i {m \choose 2i} (2i-1)!! M_{m-2i}.
\end{equation} 

Moreover, setting $\varepsilon_m= \hat{P}_m- P_m$ and $\Delta_m= \hat{M}_m- M_m,$ we have that 
 \begin{equation}
 \label{Em}
 \varepsilon_m= k\Delta_m + k \sum_{i=1}^{m/2} (-1)^{i} {m \choose 2i}(2i-1)!! \Delta_{m-2i}.
\end{equation}
 \begin{proof}[proof of lemma \ref{lemmaem}]
 
Note that $M_m$ is known as the $m$-th probabilistic Hermite polynomials \footnote{See the Wikipedia page for Hermite polynomials.}. What we prove here actually comes from the inverse explicit expression for these polynomials. We proceed with the proof using the induction. \\

We have that $kM_1= P_1(\mu_1, \cdots, \mu_k) $ and $kM_2= 1+ P_2(\mu_1, \cdots, \mu_k)$, which means $k\Delta_1=  \varepsilon_1$ and $ \varepsilon_2= k\Delta_2- k.$
  This gives us \eqref{defpm1} and \eqref{Em} for the base cases. We assume \eqref{Em} for $m-2, m-4, \cdots$ and we prove it for $m.$ \\

\noindent Using \eqref{expmom}, we have the following recursive identity relating $\varepsilon_m$ to $\Delta_m$ and $\varepsilon_{m-2}, \varepsilon_{m-4}, \cdots$. 
\begin{equation}
 \label{star}
 \varepsilon_m= k\Delta_m - \sum_{i=1}^{m/2} {m \choose 2i}(2i-1)!! \varepsilon_{m-2i}.
\end{equation}

 We use our induction hypothesis in \eqref{star} and we get everything in terms of $\Delta_m, \Delta_{m-2}, \cdots.$ Only thing remaining is to calculate the coefficients of $\Delta_{m-2i}$ for $0<i \leq m/2$. \\
 
 The term $\Delta_{m-2i}$ appears in \eqref{star} in the expansion of $\varepsilon_m, \varepsilon_{m-2}, \cdots, \varepsilon_{m-2i}$. Therefore its coefficients equal to 
\begin{equation}
 \label{meven}
 - {m \choose 2i}(2i-1)!! + \sum_{n=1}^{i-1}(-1)^{n-1} {m \choose 2n}{m-2n \choose 2i-2n}(2n-1)!!(2i- 2n-1)!!
 \end{equation}
 If $i$ is odd, then all terms cancel each other except $n=i,$ which gives $$(-1)^i {m \choose 2i}(2i-1)!!.$$ When $i$ is even we have \eqref{meven} equals to $m(m-1)\cdots(m-2i+1)$ times
 \begin{align}
 \notag  &-\frac{(2i-1)!!}{2i!} -\sum_{n=1}^{i-1}(-1)^n \frac{(2n-1)!!(2i- 2n-1)!!}{2n! (2i-2n)!} \\ & \notag = -\sum_{n=1}^{i} \frac{(-1)^n}{2^n n! 2^{i-n}(i-n)!} = -\frac{1}{2^i i!} \sum_{n=1}^{i} \frac{(-1)^n i!}{n!(i-n)!} \\ & \notag = -\frac{1}{2^i i!} \big((1-1)^i-1\big)= \frac{1}{2^i i!}.
\end{align}
Therefore, the coefficient of $\Delta_{m-2i}$ equals to $$\frac{1}{2^i i!}m(m-1)\cdots(m-2i+1)= {m \choose 2i} (2i-1)!!.$$
\\
\end{proof}
\end{lemma}
Now, if we have a large number of samples, $\gg \log(1/\delta) \epsilon^{-2}$, applying \eqref{boundonmom}, we find that $$\Delta_m \ll \epsilon \sigma^m.$$ Therefore, utilizing the above lemma, we obtain
\begin{equation}
    \label{bndpm} \varepsilon_m = \hat{P}_m - P_m \ll \epsilon k\sigma^m e^{0.5(\tfrac{m}{\sigma})^2}.
\end{equation}
We derive \eqref{bndpm} considering the Taylor expansion of $e^{0.5(\tfrac{m}{\sigma})^2}$ and comparing it to \eqref{Em}.

\section{Newton's Identities and Roots of Polynomials}

So far we established that in the case of spherical Gaussian using the method of moments we can easily calculate $P_m.$ Now consider the following polynomial: 
\begin{equation}
\label{Pco}
P(x):=\prod_{i=1}^{k}(x-\mu_i)= x^k- e_1 x^{k-1}+ \cdots+ (-1)^ke_k
\end{equation}
Recall that $P$ is our parameter polynomial. By using Newton's identities we can write coefficients $e_n,$ in terms of $P_m,$ with $m\leq n.$ For example we have 
\begin{align*}
\hspace{9 mm}& e_1= P_1 \\ &
2e_2= P^2_1- P_2 \\ & 
3e_3= \tfrac{1}{2}P^3_1 -\tfrac{3}{2}P_1P_2 + P_3, \cdots
\end{align*}
In general we have 
\begin{equation}
\label{defe}
 ne_n= \sum_{j=1}^{n}(-1)^{j-1}e_{n-j}{P}_j.
\end{equation}
Similar to the definition \ref{Phat} we define: 
\begin{defn}
Let $\hat{e}_0=1$ and $\hat{e}_1= \hat{P_1}$. Recursively we define
 \begin{equation}
 \label{defehat}
 n\hat{e}_n= \sum_{j=1}^{n}(-1)^{j-1}\hat{e}_{n-j}\hat{P}_j.
 \end{equation}
\end{defn}
By this definition and \eqref{boundonmom} and by some calculation we obtain 
\begin{align*}
& \hat{e}_1= O(\epsilon k \sigma)\\ &
2 \hat{e}_2= -k(\sigma^2-1) + O(\epsilon k \sigma^2) \\ & 
3 \hat{e}_3= K M_3 + O(\epsilon k^2 (\sigma^{3}+1)).
\end{align*}
Next we estimate how close $\hat{e}_n$ is to $e_n.$
\begin{lemma}
\label{mainlemma}
Assuming we have $\asymp \log(1/\delta) \epsilon^{-2}$ samples from our mixture, and $P_1= 0$. Then with with probability at leas $1-\delta$, and we have that 
\begin{equation}
\label{e-ehat}
 |\hat{e}_m- e_m| \ll \epsilon K(2\sigma)^me^{0.5 \big(\tfrac{m}{\sigma}\big)^2}.
\end{equation}
\end{lemma}
\begin{proof}
 Let $E_n= \hat{e}_n- e_n$. Using \eqref{defehat} and \eqref{defe}, we have that 
 \begin{equation}
 \label{eqsep}
n E_n \ll \sum_{i=1}^{n} \bigg(e_{n-i}\varepsilon_i+ E_{n-i}P_i+ E_{n-i}\varepsilon_i\bigg)
 \end{equation}
 We use following bounds on each of the terms:
\begin{enumerate}
 \item For $i \leq m$ we have $$P_i< \sigma^i.$$
 \item Using \eqref{bndpm} we have $$\varepsilon_m \ll \epsilon k\sigma^m e^{0.5(m/\sigma)^2}. $$
 \item Using induction hypothesis for $i< n$ $$E_{n-i} \ll \epsilon k (2\sigma)^{n-i} e^{0.5 \big(\tfrac{n-i}{\sigma}\big)^2}.$$
 \item From \cite{GMR} we have $$e_n \ll \big(\frac{6e}{n}\big)^{n/2} \sigma^n.$$
\end{enumerate}
\noindent Applying these bound to the first term inside the summation in \eqref{eqsep} we get
 \begin{align*}
 \sum_{i=1}^{n} & e_{n-i}\varepsilon_i \ll \epsilon \sum_{i=1}^{n} \big(\frac{6e}{n-i}\big)^{(n-i)/2} \sigma^{n-i} \sigma^i e^{0.5(i/\sigma)^2} \\ & \epsilon k \sigma^n \sum_{i=1}^{n} \big(\frac{6e}{n-i}\big)^{(n-i)/2} e^{0.5(i/\sigma)^2} \ll \epsilon k \sigma^n e^{0.5(n/\sigma)^2}.
\end{align*}
Note that for $i=1$ to $n-17$, we have $(6e/n-i)< 1.$ As for the term $$\sum_{i=1}^{n} E_{n-i}P_i$$ using the above mentioned bounds, easily gives \eqref{e-ehat}. The third error term is obviously smaller than the other two. 
\end{proof}
We prove that in the spherical case having first $k$-moment is enough to uniquely determine parameters of the mixture. 
\begin{lemma}
Let $G$ be a $k$-mixture of spherical Gaussian (KGMM):
$$G= \tfrac{1}{k}\sum_{i=1}^{k} \mathcal{N}(\mu_i, 1).$$ We have that first $k$ moments of $G$ uniquely determine parameters $\mu_1, \cdots, \mu_k.$ 
\end{lemma}
\begin{proof} By knowing $M_1, \cdots, M_k$, we can uniquely find $P_1, \cdots, P_k,$  (see Remark \ref{rem3}). Assume that $\check{M}_1 \cdots, \check{M}_k,$ would also give us $P_1, \cdots, P_k.$ Then, we have $M_1= \check{M}_1,$ by the equation \eqref{defpm1}, we have $M_2= \check{M}_2,$ and so on. \\

Since we have $P_1, \cdots, P_k,$ using Newton's identities, we can determine the coefficients of $P(x)$ in \eqref{Pco}.
This polynomial has at most $k$-distinct roots, and we know that $\mu_1, \cdots, \mu_k$ are roots of $P,$ therefore the first $k$ moments uniquely determine the parameters of the mixture. 
\end{proof}

\section{Approximating Coefficients Using Samples}
Using results in the previous section we can construct a polynomial whose roots are close to $\mu_1, \cdots, \mu_k.$ Let $$\hat{P}(x)= \sum_{n=0}^{k}\hat{e}_nx^n,$$ where we obtain $\hat{e}_n$ as in \eqref{defehat}. Using Lemma \ref{mainlemma} we have that $\hat{P}(x)$ is an approximation of $P.$ Our task is to measure how close the roots of $\hat{P}$ are to the roots of $P.$
We apply the following theorem of Beauzamy~\cite{Bea}.\\

Let $Q(x)= \sum_{i=0}^{k} a_ix^i$ be a polynomial with complex coefficients and degree k. The Bombieri’s norm of $Q$ is defined as 
\begin{equation}
 B(Q)= \Big(\sum_{i=0}^{k} \frac{|a_i|^2}{{k \choose i}}\Big)^{1/2}.
\end{equation}
\begin{theorem}[Beauzamy]
Let $P$ and $\hat{P}$ be polynomials of degree $k.$ Moreover, assume that~$B(\hat{P}-P) < \varepsilon.$ If $x$ is any zero of $P$ there exist zero $y$ of $\hat{P}$ with 
\begin{equation}
  |x-y| < \frac{k(1+ |x|^2)^{k/2}}{|\hat{P}'(x)|}\varepsilon.
\end{equation}
If $\varepsilon$ is small enough, namely 
$$\varepsilon < \frac{|P'(x)|}{k(1+ |x|^2)^{(k-1)/2}},$$
then 
\begin{equation}
|x-y| < \frac{k(1+ |x|^2)^{k/2}}{|P'(x)|}\varepsilon.
\end{equation}
\end{theorem}
\begin{Rem}
     If $P'(x)=0,$  then the theorem is void. Basically, it says that $|x-y|< \infty.$
Beauzamy~\cite{Bea} addressed this in a comment after the theorem.
\end{Rem}
\begin{proof}[Proof of Theorem \ref{1stthm}] We apply the above theorem to examine how close are roots of $\hat{P}$ to roots of $P.$ Using Lemma \ref{mainlemma} we know that if we have $\asymp \log(1/\delta) \varepsilon^{-2}$  samples from the mixture, then the Bombieri's norm of $\hat{P}- P$ is bounded by 
\begin{align*}
 B(\hat{P}- P)\ll \varepsilon \Big(\sum_{i=0}^{k} \frac{k^2e^{(i/\sigma)^2}(2\sigma)^{2i}}{{k \choose i}}\Big)^{1/2} \ll \varepsilon k e^{0.5(k/\sigma)^2}(2\sigma)^k.
\end{align*}
Therefore by Beauzamy's theorem, if $$ \varepsilon < \frac{\prod_{j \neq m} (\mu_m- \mu_j)}{k^2(1+|\mu_m|^2)^{(k-1)/2}(2\sigma)^k e^{0.5(k/\sigma)^2}}$$
there exist a root $\hat{\mu}_m$ of $\hat{P}$ such that 
\begin{equation}
\label{lsteq}
|\hat{\mu}_m- {\mu}_m| < \frac{k^2(1+|\mu_m|^2)^{k/2}(2\sigma)^k e^{0.5(k/\sigma)^2}}{\prod_{j \neq m} (\mu_m- \mu_j)} \varepsilon.
\end{equation}
This gives us the proof of Theorem \ref{1stthm}.
\end{proof}

\begin{proof}[Proof of corollary \ref{cor1}] First note that in the statement of Beauzamy's theorem and the proof of Theorem \ref{1stthm}, we used $\varepsilon$ to indicate the distance between coefficients of polynomial. Going forward we use $\epsilon$ to indicate the accuracy we expect the statement of the corollary. \\

To get Corollary \ref{cor1}, LHS of \eqref{lsteq} must be smaller than the accuracy : 
$$|\hat{\mu}_m- {\mu}_m| < \epsilon,$$ for all $m.$  

In order to have this, $\varepsilon$ in the RHS of \eqref{lsteq}, must be smaller than $\epsilon$. Therefore we must have $$ \varepsilon = \epsilon \frac{\prod_{j \neq m} (\mu_m- \mu_j)}{k^2(1+|\mu_m|^2)^{k/2}(2\sigma)^k e^{0.5(k/\sigma)^2}}.$$
Number of samples we need to get this accuracy is $$\gg \log(\delta^{-1})\varepsilon^{-2}= \log(\delta^{-1}) \epsilon^{-2} \bigg|\frac{\prod_{j \neq m} (\mu_m- \mu_j)}{k^2(1+|\mu_m|^2)^{k/2}(2\sigma)^k e^{0.5(k/\sigma)^2}} \bigg|^{-2}.$$
We use $$1+ \mu_m^2 \ll 1+ k\sigma^2,$$ and we set  $$ c(\sigma, k)= \big|k^2(1+k\sigma^{2})^{k/2}(2\sigma)^k e^{0.5(k/\sigma)^2}\big|^2$$

Therefore, the sample complexity is 
\begin{equation}
\label{sampcomp}
    \log(\delta^{-1})  \frac{c(\sigma, k)}{\big|\mathcal{P}_m\big|^{2}} \epsilon^{-2}.
\end{equation} We complete the proof of the corollary  by taking the minimum over $m$.
\end{proof}
\begin{proof}[Proof of corollary \ref{cor0}] 

We assumed $\min_{i \neq j}(\mu_i - \mu_j) > \epsilon,$ therefore applying Corollary \ref{cor1}, we have that  $\min_{m} \mathcal{P}(\mu_m)> \epsilon^{k-1}.$  This indicates that the denominator in \eqref{sampcomp} is bigger than $\epsilon^{k-1}.$ We would like to learn our parameters with $\epsilon/100$-accuracy. Therefore the number of samples we need is $$10^4\log\big(\frac{1}{\delta}\big)\epsilon^{-2k+2}{\displaystyle{ c(\sigma, k)}}\hspace{1 mm} \epsilon^{-2}.$$
This completes the proof.  
\end{proof}
\textbf{Description of the algorithm.} 
Given samples from the mixture, we first compute the empirical moments 
$\hat{M}_1, \ldots, \hat{M}_k$ as in \eqref{hatM}.  
Using \eqref{Phat}, we then recover $\hat{P}_1, \ldots, \hat{P}_k$.  
Next, applying \eqref{e-ehat}, we recursively compute $\hat{e}_1, \ldots, \hat{e}_k$, 
which give the coefficients of the empirical parameter polynomial.  
Finally, we apply a root-finding algorithm for univariate polynomials—such as the Aberth method~\cite{Al}—to approximate the component means.  
Theorem~\ref{1stthm} guarantees that the recovered roots are close to the true parameters of the mixture.

\section*{Acknowledgments} Acknowledgments. I thank Will Sawin for valuable input, and Omran Ahmadi, Hassan Ashtiani, Milad Barzegar, Moritz Hardt, and Abbas Mehrabian for their helpful comments. This work was supported by the Institute for Research in Fundamental Sciences (IPM).

\bibliographystyle{plain}
\bibliography{Stat_ML}
\bigskip
\noindent\textbf{Author Information:} \\
School of Mathematics, Institute for Research in Fundamental Sciences (IPM) \\
P.O. Box 19395-5746, Tehran, Iran \\
\texttt{aryan@ipm.ir}

\appendix
\section{Appendix}
\subsection{A Thought Experiment}

Consider the following question:

\begin{quote}
\emph{Which of the following two 7-component Gaussian mixtures is harder to learn from samples?}
\end{quote}
For the sake of our analysis, we assume that all covariance matrices in the mixture are equal to the identity, and that the means are real
. 

\textbf{Mixture A:} The means are placed so that for $i=1,2,3,4$, we have $\mu_{i+1} - \mu_i = \epsilon$, then $\mu_6 - \mu_5 = 1$, and finally $\mu_7 - \mu_6 = \epsilon$.

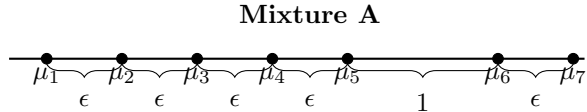
\begin{center}
\begin{tikzpicture}[scale=1]

\node at (3.5,1.6) {\textbf{Mixture A}};

\foreach \x/\label in {0/{$\mu_1$},1/{$\mu_2$},2/{$\mu_3$},3/{$\mu_4$},4/{$\mu_5$},6/{$\mu_6$},7/{$\mu_7$}} {
    \filldraw (\x,1.0) circle (2pt);
    \node[below] at (\x,1.0) {\label};
}
\draw[thick] (-0.5 ,1.0) -- (7.15,1.0);

\draw[decorate,decoration={brace,amplitude=5pt,mirror},yshift=-2pt]
(0,1.0) -- (1,1.0) node[midway,below=8pt] {$\epsilon$};
\draw[decorate,decoration={brace,amplitude=5pt,mirror},yshift=-2pt]
(1,1.0) -- (2,1.0) node[midway,below=8pt] {$\epsilon$};
\draw[decorate,decoration={brace,amplitude=5pt,mirror},yshift=-2pt]
(2,1.0) -- (3,1.0) node[midway,below=8pt] {$\epsilon$};
\draw[decorate,decoration={brace,amplitude=5pt,mirror},yshift=-2pt]
(3,1.0) -- (4,1.0) node[midway,below=8pt] {$\epsilon$};
\draw[decorate,decoration={brace,amplitude=5pt,mirror},yshift=-2pt]
(4,1.0) -- (6,1.0) node[midway,below=8pt] {$1$};
\draw[decorate,decoration={brace,amplitude=5pt,mirror},yshift=-2pt]
(6,1.0) -- (7,1.0) node[midway,below=8pt] {$\epsilon$};

\end{tikzpicture}
\end{center}

Mixture A has four consecutive $\epsilon$-sized gaps, followed by a gap of length $1$, and then another gap of length $\epsilon$. The minimum gap is $\epsilon$.

\textbf{Mixture B:} The first four means satisfy $\mu_{i+1} - \mu_i = \epsilon$ for $i=1,2,3$, then $\mu_5 - \mu_4 = 1$, $\mu_6 - \mu_5 = \epsilon$, and finally $\mu_7 - \mu_6 = \epsilon^2$.

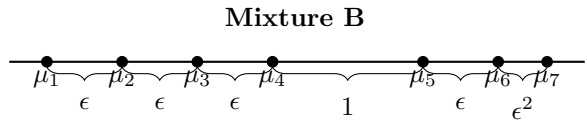
\begin{center}
\begin{tikzpicture}[scale=1]

\node at (3.3,1.6) {\textbf{Mixture B}};

\foreach \x/\label in {0/{$\mu_1$},1/{$\mu_2$},2/{$\mu_3$},3/{$\mu_4$},5/{$\mu_5$},6/{$\mu_6$},6.65/{$\mu_7$}} {
    \filldraw (\x,1.0) circle (2pt);
    \node[below] at (\x,1.0) {\label};
}
\draw[thick] (-0.5,1.0) -- (7.15,1.0);

\draw[decorate,decoration={brace,amplitude=5pt,mirror},yshift=-2pt]
(0,1.0) -- (1,1.0) node[midway,below=8pt] {$\epsilon$};
\draw[decorate,decoration={brace,amplitude=5pt,mirror},yshift=-2pt]
(1,1.0) -- (2,1.0) node[midway,below=8pt] {$\epsilon$};
\draw[decorate,decoration={brace,amplitude=5pt,mirror},yshift=-2pt]
(2,1.0) -- (3,1.0) node[midway,below=8pt] {$\epsilon$};
\draw[decorate,decoration={brace,amplitude=5pt,mirror},yshift=-2pt]
(3,1.0) -- (5,1.0) node[midway,below=8pt] {$1$};
\draw[decorate,decoration={brace,amplitude=5pt,mirror},yshift=-2pt]
(5,1.0) -- (6,1.0) node[midway,below=8pt] {$\epsilon$};
\draw[decorate,decoration={brace,amplitude=5pt,mirror},yshift=-2pt]
(6,1.0) -- (6.65,1.0) node[midway,below=8pt] {$\epsilon^2$};

\end{tikzpicture}
\end{center}
Mixture B has three consecutive gaps of length $\epsilon$, followed by a gap of length $1$, then another gap of length $\epsilon$, and finally a much smaller gap of length $\epsilon^2$. The minimum gap in this configuration is therefore $\epsilon^2$.\\

Back to our question: Which mixture is harder to learn from i.i.d. samples? Many would instinctively choose Mixture B, as it contains a dramatically smaller gap. In fact, when two Gaussian components have equal variances and their means are close, they become statistically hard to distinguish.\\

The same intuition appears in much of the literature, where sample complexity is commonly attributed to the minimum distance between component means. For example, Yang and Wu~\cite{annstat} derive sample complexity bounds that scale with inverse powers of $g_{\min}$, the minimum pairwise gap; see also \cite{MGA, Lili}. From this perspective, the $\epsilon^2$-sized gap in Mixture B suggests it should be the more difficult case.\\

Our analysis shows that this perception is incorrect: Mixture B is actually easier to learn.  
The reason is that the \emph{Pair Correlation Factor (PCF)} for Mixture A is substantially smaller.  
While Mixture B has one extremely small gap $\epsilon^2$, its other gaps are much larger, so the product 
$\mathcal{P}(\mu_m)$ remains relatively large.  
In contrast, Mixture A contains several consecutive $\epsilon$-sized gaps, causing the PCF to shrink to the order of $\epsilon^4$ (or smaller), which drives up the required sample complexity.  
Thus, despite the smaller $g_{\min}$ in Mixture B, it is Mixture A that is statistically harder to learn.

\end{document}